\documentclass{article}
\pdfoutput=1

\usepackage[final, nonatbib]{neurips_2022}

\usepackage[utf8]{inputenc} 
\usepackage[T1]{fontenc}    
\usepackage{hyperref}       
\usepackage{url}            
\usepackage{booktabs}       
\usepackage{amsfonts}       
\usepackage{nicefrac}       
\usepackage{microtype}      
\usepackage{xcolor}         
\usepackage{amsmath}
\usepackage{graphicx}
\usepackage{subcaption}
\usepackage{algorithm}
\usepackage{setspace}
\usepackage{algpseudocode}
\usepackage{multirow}
\usepackage{wrapfig}
\usepackage{makecell} 
\usepackage{caption}
\usepackage{adjustbox}
\usepackage{amsthm}
\usepackage[square,numbers]{natbib}
\theoremstyle{plain}
\newtheorem{thm}{Theorem}
\newtheorem{obs}{Observation}

\theoremstyle{definition}
\newtheorem{defn}{Definition}
\definecolor{PG}{RGB}{20, 160, 40}
\definecolor{Purple}{RGB}{140, 30, 180}

\title{DUEL: Adaptive  Duplicate Elimination on \\Working Memory for Self-Supervised Learning}

\author{%
  Won-Seok Choi$^1$, Dong-Sig Han$^1$, Hyundo Lee$^1$, Junseok Park$^1$, Byoung-Tak Zhang$^{12}$\\
  $^1$Seoul National University, $^2$AIIS\\
  \texttt{\{wchoi,dshan,hdlee,jspark,btzhang\}@bi.snu.ac.kr} \\
}

\begin{document}

\maketitle

\begin{abstract}
In Self-Supervised Learning (SSL), it is known that frequent occurrences of the \textit{collision} in which target data and its negative samples share the same class can decrease performance. Especially in real-world data such as crawled data or robot-gathered observations, collisions may occur more often due to the duplicates in the data. To deal with this problem, we claim that sampling negative samples from the adaptively debiased distribution in the memory makes the model more stable than sampling from a biased dataset directly. In this paper, we introduce a novel SSL framework with adaptive Duplicate Elimination (DUEL) inspired by the \textit{human working memory}. The proposed framework successfully prevents the downstream task performance from degradation due to a dramatic inter-class imbalance.
\end{abstract}

\section{Introduction}
\label{intro}

\begin{wrapfigure}{r}{0.40\textwidth}
    \vspace*{-4em}
    \begin{center}
    \includegraphics[width=0.40\textwidth]{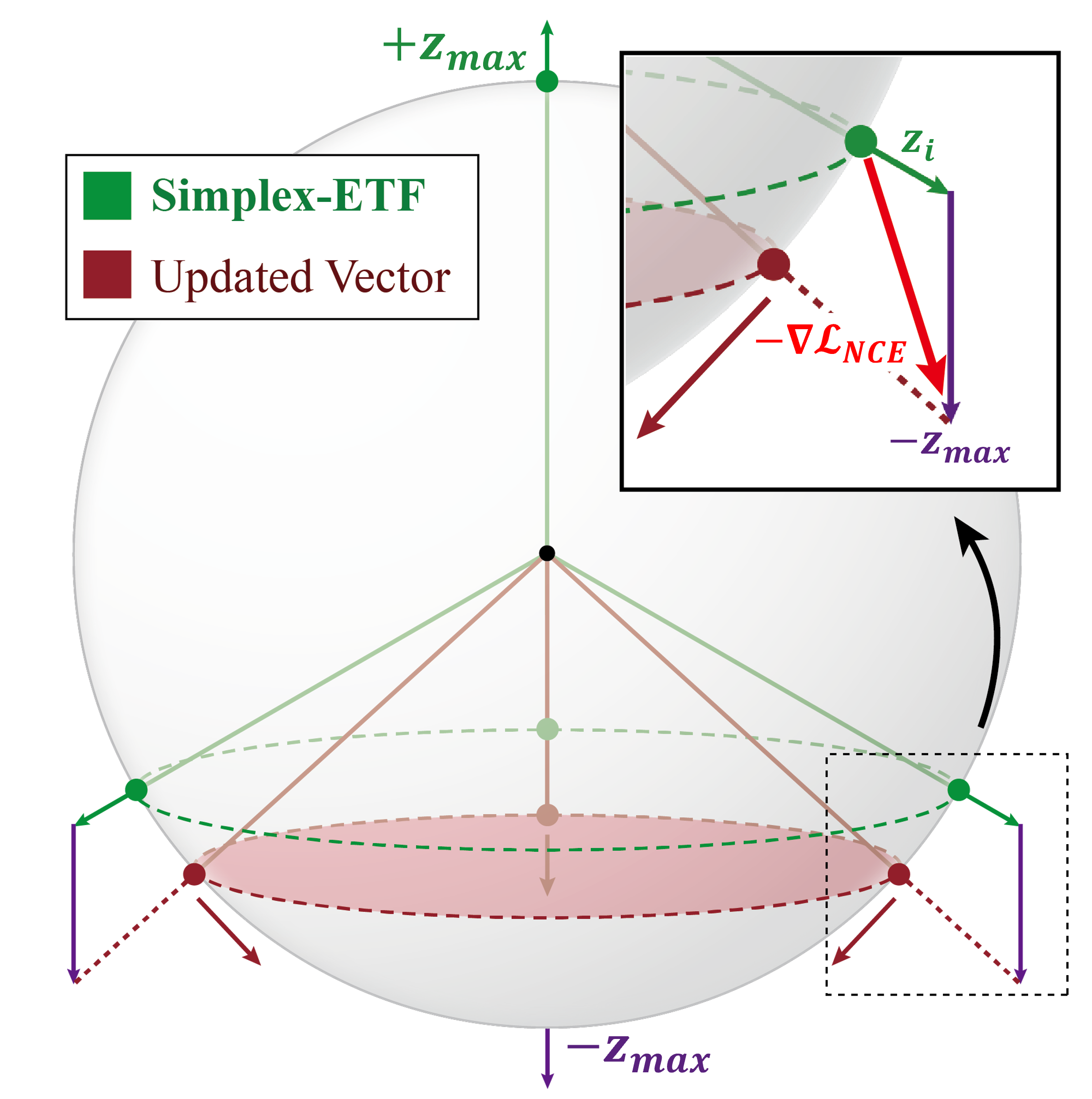}
    \vspace*{-1.5em}
    \caption{\small Visualization of \textbf{Observation \ref{obs1}}. If the dataset is biased with a class $c_{\max}$, representations of other classes (\textcolor{PG}{green}) come closer to the opposite of $z_{\max}$ (\textcolor{Purple}{purple}).}
    \label{fig1}
    \end{center}
    \vspace*{-2.5em}
\end{wrapfigure}
In Self-Supervised Learning (SSL), there is a possibility that target data and its negative samples' class information are partially duplicated during the sampling process. This phenomenon is called \textit{collision} and it leads to the degradation of the latent space's representability \citep{ash2022investigating,awasthi2022more}. When data is provided in the real world such as crawled images from the web or robot-gathered vision data, the agent may face many duplicates and they can cause collisions frequently.

\textit{Human working memory} \citep{baddeley1992working,miyake2000unity,wongupparaj2015relation} has a Central Executive System (CES) which manages the limited memory efficiently to enhance the task performance. Some of the major roles of CES are as follows: \textit{inhibition of dominant signals} and \textit{updating recent signal} in memory \citep{miyake2000unity}. In this paper, we claim that a more efficient memory with an adaptive controller akin to the human working memory is essential to reduce the collisions for more robust training with a biased dataset.

In this paper, we first evaluate the previous SSL frameworks' robustness when the data is highly biased with a specific class. Based on the results, we introduce a novel SSL framework with adaptive Duplicate Elimination (DUEL) to imitate human working memory. The proposed framework performs training of the feature extractor and removing the most duplicated data from the memory with the current feature extractor simultaneously. We compare our proposed framework to previous popular frameworks in a biased dataset adapted from the common vision dataset.

\section{Related Work}
\label{prev_works}
The main difference between the Self-Supervised Learning frameworks is the method of selecting the negative samples. SimCLR \citep{pmlr-v119-chen20j, chen2020big} uses other data in the same batch as negative samples. MoCo \citep{He_2020_CVPR,chen2020improved} has the external memory to store representative information. BYOL \citep{NEURIPS2020_f3ada80d} and SimSiam \citep{Chen_2021_CVPR} can be trained with only positives for training by using bootstrapping.

Recently, there were analyses on the relationship between Noise Contrastive Estimation (NCE) loss and supervised loss. \citet{ash2022investigating} found that the upper bound of the supervised learning loss can be derived with two different terms: NCE loss without collisions and intra-class variances. They claimed when the collision occurs frequently, it may increase the intra-class variance and loosen the upper bound. \citet{awasthi2022more} followed the formulas from \citet{ash2022investigating} and proved the representations which optimizes the NCE loss form the simplex \textit{Equiangular Tight Frame (ETF)}.
\begin{defn} [Simplex-ETF]
The normalized representations $z_{1},\cdots,z_{k}$ and their classes $c_{1},\cdots,c_{k}\in\mathcal{C}$ form simplex-ETF when they satisfy the following property. 
\vspace{-1em}
\end{defn}
\begin{equation}
\forall i,j,\: z_{i}^Tz_{j}=
\begin{cases}
1\qquad\qquad c_{i}=c_{j}\\
-1/|\mathcal{C}|\quad\:\:\, c_{i}\neq c_{j}
\end{cases}
\end{equation}
In this work, we used the simplex-ETF to analyze the robustness of SSL frameworks.

\section{Analysis on InfoNCE Loss with Biased Dataset}
\label{ssl_with_biased}

 In this section, we analyze which shape the representations form to optimize the NCE loss with a biased dataset by expanding previous works' approaches \citep{ash2022investigating,awasthi2022more,NEURIPS2020_d89a66c7}. Let a dataset $\mathcal{D}$ contain a pair of data $x$ and its implicit class $\hat{c}$, $d=(x,\hat{c})\sim\mathcal{D}$. The implicit class $\hat{c}$ has its marginal distribution $\hat{c} \sim \rho$. In general, the positive sample $d^+$ has the same class $\hat{c}$ as $d$ and negative samples $d^-_{1:k}$ are drawn in the i.i.d. manner from the same distribution $\rho$. NCE loss is derived as below.
\begin{equation}
\mathcal{L}_{\text{NCE}}(f)=\mathbb{E}_{d,d^+,d^-_{1:k}}\left[\ell\left(\{f(x)^T\left(f(x^+)-f(x^-_{i})\right)\}_{i=1}^{k}\right)\right]
\end{equation}
$f$ is the feature extractor which projects the data $x$ onto a hypersphere. $\ell$ is a logistic loss function $\ell(\pmb{v})=\log(1+\sum \exp(-v_i))$ which is widely used recently.
\subsection{Gradient of Representation on NCE Loss with Biased Distribution}
\label{comp_grad}
Let a class $c_{\max}$ occurs more frequently than others. Then the probability of choosing $c_{\max}$ is $\rho_{\max}$, and otherwise $\rho_{\min}=\frac{1-\rho_{\max}}{|\mathcal{C}|-1}$. In practice, representations of data with each class form the clusters whose mean vectors represent the ETF-like shape. If mean vectors get closer to each other, the chance of overlapping among them will increase and it will decrease the downstream task performance.

\begin{obs}[Non-convergence to simplex-ETF with biased data] \label{obs1}
The representations optimized by the NCE loss will not converge to the simplex-ETF when the data is biased with class $c_{\max}\in \mathcal{C}$.
\end{obs}

We compute the gradient of the NCE loss with respect to each vector on the simplex-ETF with the equation in \citet{NEURIPS2020_d89a66c7}. The gradient of each case is computed as Equation \ref{eqn-res-thm1}.

\begin{equation}
\frac{\partial \mathcal{L}_{\text{NCE}, i}}{\partial z_{i}}\propto\begin{cases}z_i\cdot (-1+\rho_{\max}-\rho_{\min})& c_{i}=c_{\max}\\
-z_{i}+z_{\max}\cdot (\rho_{\max}-\rho_{\min}) & c_{i}\neq c_{\max}
\end{cases}
\label{eqn-res-thm1}
\end{equation}

As a result, the gradient will contain a non-zero $z_{\max}$ term when $c_{i}\neq c_{\max}$. This means that all vectors except $z_{\max}$ will get closer to the $-z_{\max}$ direction after updates. This indicates that frequent collision will disturb the training of general representation of the data. Details on this observation is provided in Appendix \ref{apdx-obs3}. The visualization is also shown in Figure~\ref{fig1}.

\begin{figure*}[tb]
    \begin{center}
        \includegraphics[width=\textwidth]{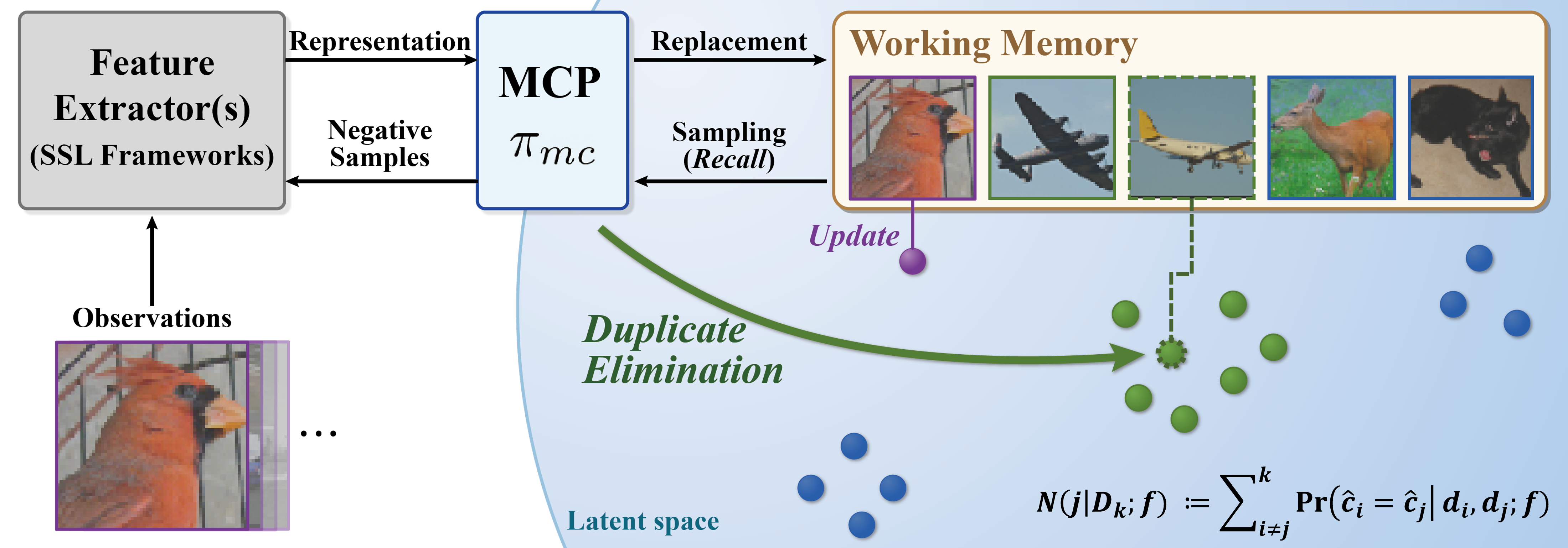}
        \caption{Visualization of general DUEL framework. Our method stores various data for the negative samples by adaptive Duplicate Elimination. A memory control policy selects the most duplicated sample in memory (\textcolor{PG}{green}) and replaces it with current data (\textcolor{Purple}{purple}).}
        \label{figDUEL}
    \end{center}
    \vspace*{-0.5em}
\end{figure*}

\section{Adaptive Duplicate Elimination (DUEL) with Working Memory}
\label{DUEL}
For real-world agents such as humans, dealing with biases caused by physical accessibility is important. The working memory solves the problem by updating recent data while reducing the intensity of dominant information \citep{miyake2000unity,wongupparaj2015relation}. Inspired by this paradigm, we propose a memory control policy that replaces duplicated data with current data to reduce biases in the dataset.

\subsection{Memory Control Policy (MCP)}
\begin{wrapfigure}{r}{0.4\textwidth}
    \vspace*{-4em}
    \begin{center}
    \includegraphics[width=0.4\textwidth]{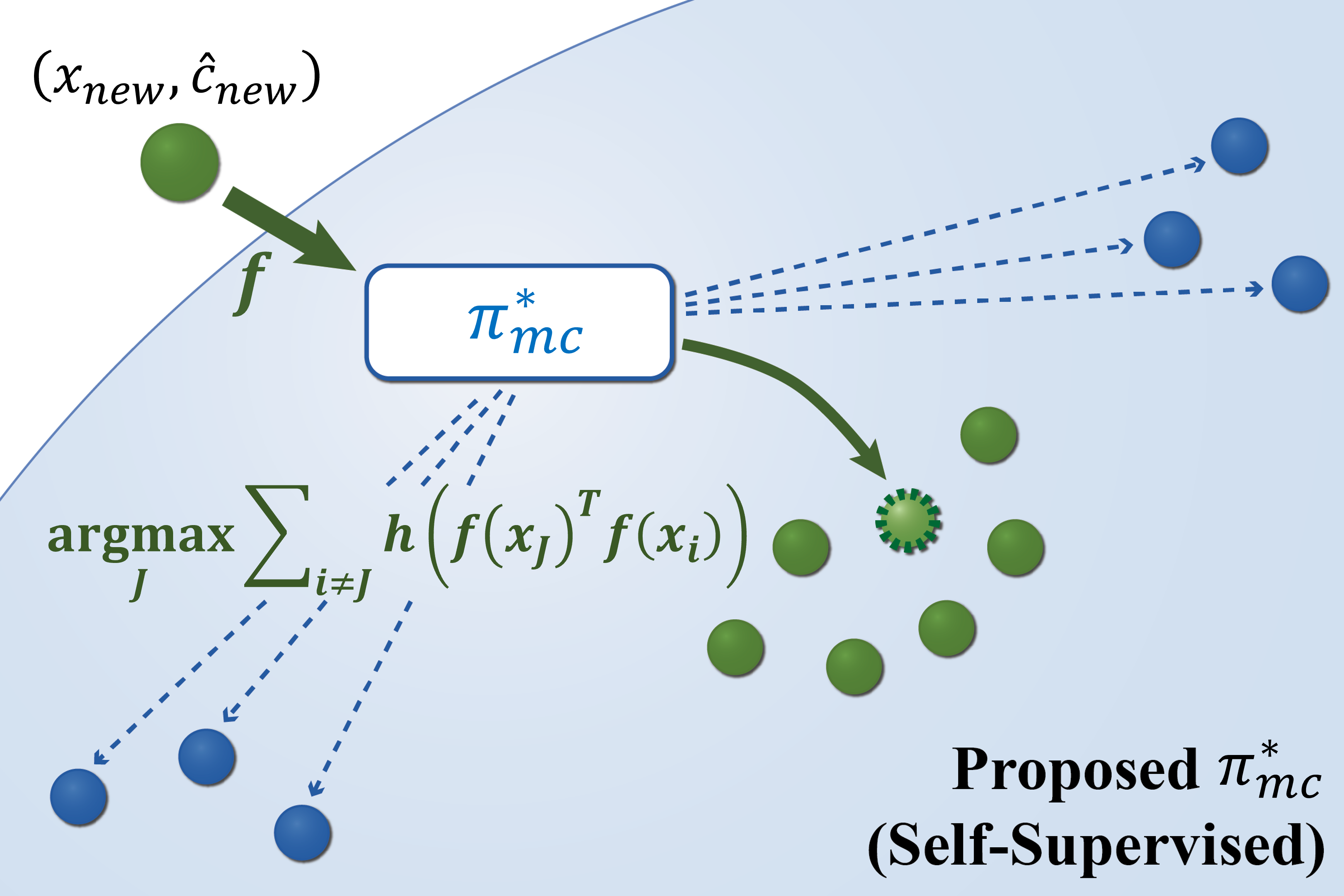}
    \vspace*{-1.5em}
    \caption{$\pi_{mc}^*$ in SSL.}
    \label{fig2b}
    \end{center}
    \vspace*{-3em}
\end{wrapfigure}
In the perspective of working memory, both storing recent data and avoiding the dominance of specific information are essential. This implies that the MCP should focus on choosing the most duplicated data for replacement and also needs a measurement of information to define the replacement criterion. In this case, the agent only can get the cosine similarity of data pairs in SSL, so we define a pair-wise collision probability with a score function to measure the estimated number of duplicates in memory.

\begin{defn}[Pair-wise collision probability]
Let there be two data $d_i=(x_{i},\hat{c}_{i})$ and $d_j=(x_{j},\hat{c}_{j})$. The probability that the implicit classes $\hat{c}_i$ and $\hat{c}_j$ are the same with feature extractor $f$ and a score function $h$ can be defined as below.
\vspace{-1.5em}
\end{defn}

\begin{equation}
\label{eqn-collision}
Pr(\hat{c}_i=\hat{c}_j|d_i,d_j;f)=h(f(x_i)^Tf(x_j))
\end{equation}

There can be various $h$ functions that satisfy the properties of cosine similarity and probability. This pair-wise collision probability can be applied to the MCP to compute the expected number of duplicates of each sample in a set of data $D_k=\{d_i\}_{i=1}^{k}$ with Equation \ref{eqn-n-dup}.
\begin{equation}
\label{eqn-n-dup}
N(j|D_k;f):=\sum_{i\neq j}^{k}Pr(\hat{c}_i=\hat{c}_j|d_i,d_j;f)
\end{equation}
 With $N(\cdot|D_k;f)$ function, we can design a simple MCP to remove samples that have duplicated information in the memory.
\begin{defn}[Naïve MCP]
Let there be a poilcy $\pi_{mc}^*$ that handles data one by one. In replacement, the policy chooses $J$-th sample as $\arg\max_J N(J|D_k;f)$ and replaces it with the current input.
\vspace{-0.5em}
\label{defn:MCP}
\end{defn}

Figure \ref{fig2b} shows the behavior of proposed MCP $\pi_{mc}^*$. The $\pi_{mc}^*$ finds the dense area (green) of the latent space and ejects the most duplicated element (dotted outline). The sparse region (blue) is not influenced by this replacement and it will increase or maintain the variety of the memory structure.

\subsection{DUEL Framework for Biased Dataset}
\label{DUEL_framework}
In DUEL framework, adaptive duplicated elimination with the model $f$ and training of the model parameters $\theta$ are executed simultaneously. The procedure of our framework is shown in Figure \ref{figDUEL}. The framework shares the training part with previous works, so main difference of our framework is choosing the most duplicated data in memory $\mathcal{M}$. For experiments, we use the DUEL framework with $\pi_{mc}^*$ in Definition \ref{defn:MCP} which selects $J$-th element with the highest value of $N(J|\mathcal{M};f,\theta)$. We use the simple policy that processes each data on-the-fly, but there also can be more complex policies from such as reinforcement learning that can update data in a batch at once. General DUEL algorithm and our implementation with $\pi_{mc}^*$ are shown in Appendix \ref{apdx-alg}. 

\section{Experiment}

The DUEL's ultimate goal is both effective and efficient framework which can remove duplicates in memory adaptively and train its feature extractor at the same time. We design experiments to validate our approach's robustness in the biased environment.

\begin{table}[tb]
    \renewcommand*{\arraystretch}{1.1}
    \begin{center}
        \caption{Top-$k$ accuracy. (CIFAR-10, \%)}
        \label{table1}
        \begin{adjustbox}{width=\textwidth}
            \begin{tabular}{c|cccc|cccc}
            \Xhline{3\arrayrulewidth}
            \multirow{3}{*}{Methods}&\multicolumn{4}{c|}{Top-1} &\multicolumn{4}{c}{Top-10} \\
            &\multicolumn{4}{c|}{Bias factor ($\rho_{\max}/\rho_{\min}$)}&\multicolumn{4}{c}{Bias factor ($\rho_{\max}/\rho_{\min}$)} \\
            &$\times1.0$&$\times3.0$&$\times9.0$&$\times27.0$&$\times1.0$&$\times3.0$&$\times9.0$&$\times27.0$\\
            \hline
            MoCoV2\citep{chen2020improved}&31.38&32.38&27.68&23.78&84.55&84.32&81.74&78.40\\
            SimCLR\citep{pmlr-v119-chen20j}&39.85&36.94&35.58&27.17&88.38&87.20&84.34&76.98\\
            SimSiam\citep{Chen_2021_CVPR}&18.55&16.14&18.72&20.11&76.45&73.58&76.58&78.91\\
            \hline
            D-MoCo (ours)&37.12&34.45&\bf38.20&28.13&87.22&85.48&\bf87.59&\bf82.99\\
            D-SimCLR (ours)&\bf 42.51&\bf40.79&37.14&\bf33.36&\bf89.33&\bf88.66&82.49&80.51\\
            \Xhline{3\arrayrulewidth}
            \end{tabular}
        \end{adjustbox}
    \end{center}
    \vspace{-1em}
\end{table}

\begin{table}
    \renewcommand*{\arraystretch}{1.1}
    \begin{minipage}{0.50\textwidth}
        \vspace{-7em}
        \caption{Downstream task accuracy. (\%)}
        \label{table2}
        \begin{adjustbox}{width=\textwidth}
            \begin{tabular}{c|cccc}
            \Xhline{3\arrayrulewidth}
            \multirow{2}{*}{Methods}&\multicolumn{4}{c}{Bias factor ($\rho_{\max}/\rho_{\min}$)} \\
            &$\times1.0$&$\times3.0$&$\times9.0$&$\times27.0$\\
            \hline
            MoCoV2&54.25&55.32&49.54&42.17\\
            SimCLR&63.34&61.19&61.06&49.02\\
            \hline
            \multirow{2}{*}{D-MoCo}&58.15&56.59&\bf62.10&52.06\\
            &\textcolor{PG}{\textbf{(+3.9)}}&\textcolor{PG}{(+1.27)}&\textcolor{PG}{(+12.56)}&\textbf{\textcolor{PG}{(+9.89)}}\\
            \hline
            \multirow{2}{*}{D-SimCLR}&\bf65.39&\bf63.69&60.99&\bf56.60\\
            &\textcolor{PG}{\textbf{(+2.05)}}&\textcolor{PG}{(+2.50)}&(-0.07)&\textbf{\textcolor{PG}{(+7.58)}}\\
            \Xhline{3\arrayrulewidth}
            \end{tabular}
        \end{adjustbox}
    \end{minipage}%
    \hspace{1em}
    \vspace{-2em}
    \begin{minipage}{0.45\textwidth}
        \begin{adjustbox}{width=\textwidth}
            \includegraphics{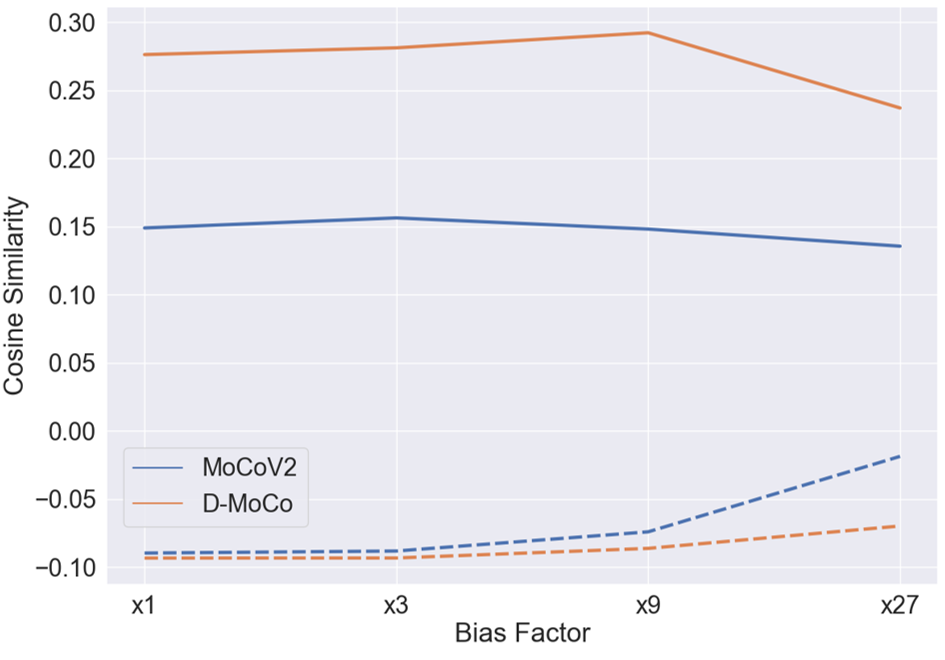}
        \end{adjustbox}
        \captionof{figure}{Average intra- (solid) and inter-class (dotted) cosine similarity. (MoCo-based)}
        \label{fig3}
    \end{minipage}
    \vspace{-2.5em}\\
    \begin{minipage}{0.50\textwidth}
        \vspace{-1.5em}
        \begin{center}
            \caption{\textbf{Ablation:} $h$ function. ($\times27.0$, \%)}
            \label{table3}
            \begin{adjustbox}{width=0.85\textwidth}
                \begin{tabular}{c|ccc}
                \Xhline{3\arrayrulewidth}
                Methods&Linear&Gaussian&Quadratic\\
                \hline
                D-MoCo&52.06&56.81&57.32\\
                D-SimCLR&56.60&-&-\\
                \Xhline{3\arrayrulewidth}
                \end{tabular}
            \end{adjustbox}
        \end{center}
    \end{minipage}
\end{table}
\textbf{Experiment setting.} We implement the artificial dataset with biased class distribution $\rho$ in Section \ref{comp_grad}: $\rho_{\max}$ with $c_{\max}$ and otherwise, $\rho_{\min}$. Both proposed and previous frameworks are compared with various bias factors $\rho_{\max}/\rho_{\min}$. More details on the training are explained in Appendix \ref{apdx-detail-5}.

\textbf{Comparison with competitive frameworks.} We first evalute our approach to the popular SSL frameworks. Table \ref{table1} shows the Top-$k$ accuracies of each model. Our approaches outperform the previous frameworks in all environments with various bias factors. There are large gains of about 5\% in average with the bias factor $\times$27 in both D-MoCo and D-SimCLR than their origin framework respectively. Table \ref{table2} also shows our framework is robust in \textit{every} environment including the unbiased dataset. It implies that the processed dataset still contains \textit{implicit} biases which humans hardly discriminate, and our framework can remove them effectively. Figure \ref{fig3} also supports proposed framework can extract more robust representations than original framework: average inter-class cosine similarity of proposed model is smaller, and conversely, average intra-class similarity is bigger.

\textbf{Ablation study: $h$ function.} We test our framework with several $h$ functions: Linear (default), Quadratic, and Gaussian. The visualizations of these functions are also shown in Appendix \ref{apdx-detail-5}. MoCo-based model performs robustly with all function types, but SimCLR-based one fails when the function is not linear. We carefully claim that this function will affect the behavior of MCP in early-stage learning, but more analysis is needed to prove this phenomenon.

\section{Discussion}

In this work, we suggest a novel framework with adaptive memory control which imitates the working memory's behavior. Our framework outperforms its original SSL framework, especially in a severely biased dataset. The robustness of proposed framework supports our claim that adaptive methods in gathering negative samples are essential for more stable SSL training.

\begin{ack}
This work was partly supported by the Institute of Information \& Communications Technology Planning \& Evaluation (2015-0-00310-SW.StarLab/20\%, 2019-0-01371-BabyMind/20\%, 2021-0-02068-AIHub/10\%, 2021-0-01343-GSAI/10\%, 2022-0-00951-LBA/20\%, 2022-0-00953-PICA/20\%) grant funded by the Korean government.
\end{ack}

\bibliography{references.bib}

\appendix
\setcounter{thm}{0}
\setcounter{obs}{0}
\section{Appendix}

\subsection{Algorithms}
\label{apdx-alg}

Algorithm \ref{alg-DUEL-general} and Algorithm \ref{alg-DUEL-naive} describe our proposed framework, called DUEL. The replacement in the working memory can be interpreted with the decision-making problem. In this case, the decisions with the policy are the indices of the data in the working memory. Various algorithms which can be separated into the smaller decision-making processes (e.g. sorting or Reinforcement Learning) can be used for the replacement procedure. Also, there is an alternative way to implement our framework: updating memory with current data first and sampling only from memory for the training. However, we do not conduct experiments with those setups because we avoid considerable modifications to the previous frameworks for a fair comparison. In this work, we have shown that our framework is general and also performs more robust than its original model, so experiments with those various techniques will be conducted in future work.

\begin{algorithm}[h]
    \setstretch{1.1}
    \begin{algorithmic}[1]
        \State \textbf{Model : } feature extractor $f$, memory $\mathcal{M}$, memory control policy $\pi_{mc}$
        \State \textbf{Input : } biased dataset $\mathcal{D}$, batch size $B$, learning rate $\eta$, intial memory $\mathcal{M}_{0}$
        \State \textbf{Output :} Trained parameter $\theta^*$
        \State $\theta\leftarrow \theta_{0}$
        \State $\mathcal{M}\leftarrow \mathcal{M}_0$
        \While {$\theta$ is not converged}
            \State $\{(d_{b},d_{b}^+)\}_{b=1}^B \leftarrow \text{sample}(\mathcal{D})$
            \State $\mathcal{L} \leftarrow 
            \mathcal{L}_{\text{NCE}}(\{d_{b}\}_{b=1}^B,\{d_{b}^+\}_{b=1}^B,\{d_{b}^+\}_{b=1}^B\cup\mathcal{M};f,\theta)$
            \State $\theta\leftarrow \theta - \text{Optimizer}(\nabla_{\theta}\mathcal{L}, \eta)$
            \For {$b \in \{1,\cdots,B\}$}
                \State $\mathcal{M} \leftarrow \mathcal{M}\cup\{d_{b}\}$
                \State $J \leftarrow \arg\max_{J}\pi_{mc}(\mathcal{M},J;f,\theta)$
                \State $\mathcal{M}\leftarrow \mathcal{M}/\{d_{J}\}$
            \EndFor
            \State $\pi_{mc} \leftarrow \text{UpdatePolicy}(\pi_{mc})$
        \EndWhile
        \State $\theta^*\leftarrow \theta$
    \end{algorithmic}
    \caption{DUEL Framework}
    \label{alg-DUEL-general}
\end{algorithm}

\begin{algorithm}[h]
    \setstretch{1.1}
    \begin{algorithmic}[1]
        \State \textbf{Model : } feature extractor $f$, memory $\mathcal{M}$
        \State \textbf{Input : } biased dataset $\mathcal{D}$, batch size $B$, learning rate $\eta$, intial memory $\mathcal{M}_{0}$
        \State \textbf{Output :} Trained parameter $\theta^*$
        \State $\theta\leftarrow \theta_{0}$
        \State $\mathcal{M}\leftarrow \mathcal{M}_0$
        \While {$\theta$ is not converged}
            \State $\{(d_{b},d_{b}^+)\}_{b=1}^B \leftarrow \text{sample}(\mathcal{D})$
            \State $\mathcal{L} \leftarrow 
            \mathcal{L}_{\text{NCE}}(\{d_{b}\}_{b=1}^B,\{d_{b}^+\}_{b=1}^B,\{d_{b}^+\}_{b=1}^B\cup\mathcal{M};f,\theta)$
            \State $\theta\leftarrow \theta - \text{Optimizer}(\nabla_{\theta}\mathcal{L}, \eta)$
            \For {$b \in \{1,\cdots,B\}$}
                \State $J \leftarrow \arg\max_J N(J|\mathcal{M};f,\theta)$ \text{in Equation \ref{eqn-n-dup}} \Comment{naïve MCP $\pi_{mc}^*$}
                \State $\mathcal{M}\leftarrow (\mathcal{M}\cup\{d_{b}\})/\{d_{J}\}$
            \EndFor
        \EndWhile
        \State $\theta^*\leftarrow \theta$
    \end{algorithmic}
    \caption{DUEL Framework with naïve MCP}
    \label{alg-DUEL-naive}
\end{algorithm}

\newpage
\subsection{Proofs}
\label{apdx-obs3}
\begin{obs}[Non-convergence to simplex-ETF with biased dataset] \label{obs1}
Let the representation of data form simplex-ETF. Then representations will not converge to the simplex-ETF when the data is biased with specific class $c_{\max}$.
\vspace{-0.5em}
\end{obs}
\begin{proof}
    In \citet{NEURIPS2020_d89a66c7}, the authors formulated the gradient of the NCE Loss. We ignored the temperature parameter $\tau$ for the ease of formulation.
    \begin{equation}
    \label{eqn:grad-nce}
    \frac{\partial \mathcal{L}_{i}}{\partial z_{i}}=\sum_{p\in P(i)}z_{p}(P_{ip}-X_{ip})+\sum_{n\in N(i)}z_{n}P_{in}
    \end{equation}
    
    Let $c_{i}=c_{max}$. In this case, $P_{ip}$, $X_{ip}$ and $P_{in}$ can be computed as below. Note that the $P_{ip}$ and $P_{in}$ is the same with $p$ and $n$ each.
    \begin{equation}
    P^+=P_{ip}=\frac{exp(1)}{k\{\rho_{\max}exp(1)+(1-\rho_{\max})exp(-1/|\mathcal{C}|)\}}, X_{ip}=\frac{1}{k\rho_{\max}}
    \end{equation}
    \begin{equation}
    P^-=P_{in}=\frac{exp(-1/|\mathcal{C}|)}{k\{\rho_{\max}exp(1)+(1-\rho_{\max})exp(-1/|\mathcal{C}|)\}}
    \end{equation}
    Then the gradient of $z_i$ can be derived as
    \begin{equation}
    \frac{\partial \mathcal{L}_{i}}{\partial z_{i}}=k\rho_{\max}z_{i}\left(P^+-\frac{1}{k\rho_{\max}}\right)+P^-\cdot\sum_{n\in N(i)}z_{n}
    \end{equation}
    By the property of the ETF, the sum of all different vectors is $0$.
    \begin{equation}
    =z_{i}\cdot k\rho_{\max}\left(P^+-\frac{1}{k\rho_{\max}}\right)-z_{i}\cdot k\rho_{\min}P^-
    \end{equation}
    \begin{equation}
    =z_i(k\rho_{\max}P^+-k\rho_{\min}P^--1)=z_i\cdot kP^-(-1+\rho_{\max}-\rho_{\min})
    \end{equation}
    So the gradient of the loss function is parallel to the feature vector when $c_{i}=c_{\max}$. However, if $c_{i}\neq c_{\max}$, the gradient will be computed differently.
    \begin{equation}
    P^+=P_{ip}=\frac{exp(1)}{k\{\rho_{\min}exp(1)+(1-\rho_{\min})exp(-1/|\mathcal{C}|)\}}, X_{ip}=\frac{1}{k\rho_{\min}}
    \end{equation}
    \begin{equation}
    P^-=P_{in}=\frac{exp(-1/|\mathcal{C}|)}{k\{\rho_{\min}exp(1)+(1-\rho_{\min})exp(-1/|\mathcal{C}|)\}}
    \end{equation}
    \begin{equation}
    \frac{\partial \mathcal{L}_{i}}{\partial z_{i}}=k\rho_{\min}z_{i}\left(P^+-\frac{1}{k\rho_{\min}}\right)+P^-\cdot\sum_{n\in N(i)}z_{n}
    \end{equation}
    \begin{equation}
    =z_{i}(k\rho_{\min}P^+-1-k\rho_{\min}P^-)+z_{\max}\cdot k(\rho_{\max}-\rho_{\min})P^-
    \end{equation}
    \begin{equation}
    =z_{i}\cdot(-kP^-)+z_{\max}\cdot kP^-(\rho_{\max}-\rho_{\min})
    \end{equation}
\end{proof}

\begin{thm}[Safety of naïve MCP]
Let the latent space form a simplex ETF. Suppose that there is a new data $d_{new}$ and $\pi_{mc}^*$ should replace a sample with $d_{new}$. Let sample chosen by $\pi_{mc}^*$ among the pool $D_k$ be $d_{J}$. Let $p_d$ with the replaced pool be $p_{d|\pi}$. After then replacement, $p_d\le p_{d|\pi}$ is satisfied for any pool $D_k$.
\vspace{-0.5em}
\end{thm}
\begin{proof}
    \begin{equation}
    p_d=\sum_{i}\sum_{j}Pr(\hat{c}_{i}\neq \hat{c}_{j}|d_{i},d_{j};f)
    \end{equation}
    \begin{equation}
    \label{eqn:pd-change}
    =\sum_{i}\sum_{j}\{1-Pr(\hat{c}_{i}= \hat{c}_{j}|d_{i},d_{j};f)\}
    \end{equation}
    
    To show $p_{d}\le p_{d|\pi}$, we show that $p_{d}$ increases when $d_{new}$ replaces $d_{J}$. During replacement, $(k-1)$ samples will not be changed, so we defined $\Delta p_d$ and $\Delta p_{d|\pi}$ to ignore the remaining ones' relationship. In the end, proving $p_{d}\le p_{d|\pi}$ is the same as proving $\Delta p_{d} \le \Delta p_{d|\pi}$.
    
    \textit{i) $\hat{c}_{\max}=\hat{c}_{new}$}
    
    Let $c_{\max}$ be the implicit class with the maximum number of samples that share the representation. By the definition, $\arg\max_J N(J|D_k;f)=\sum_{i}Pr(\hat{c}_{i}= \hat{c}_{J}|d_{i},d_{j};f)$ selects an element with the class $\hat{c}_{\max}$. $\hat{c}_{\max}=\hat{c}_{new}$ means that two data $d_{new}$ and $d_{J}$ are \textit{latent indistinguishable}, so $\Delta p_{d}=\Delta p_{d|\pi}$.\\
    
    \textit{ii) $\hat{c}_{\max}\neq\hat{c}_{new}$}
    
    In this case, $n_{\hat{c}_{J}}=n_{\hat{c}_{\max}}\ge n_{\hat{c}_{new}}$ is satisfied.
    
    \begin{equation}
    \Delta p_{d|\pi} = \sum_{i\neq J} \{1 - Pr(\hat{c}_{i}=\hat{c}_{new}|\cdot)\} + 1 - Pr(\hat{c}_{new}=\hat{c}_{new}|\cdot)
    \end{equation}
    \begin{equation}
    = (n_{\hat{c}_{new}}-1)\{1 - h(-1/|\mathcal{C}|)\} + (k-n_{\hat{c}_{new}})\{1 - h(1)\} + 1 - Pr(\hat{c}_{J}=\hat{c}_{J}|\cdot)
    \end{equation}
    \begin{equation}
    \ge (n_{\hat{c}_{J}}-1)\{1 - h(-1/|\mathcal{C}|)\} + (k-n_{\hat{c}_{J}})\{1 - h(1)\} + 1 - Pr(\hat{c}_{J}=\hat{c}_{J}|\cdot)
    \end{equation}
    \begin{equation}
    =\sum_{i\neq J}\{1 - Pr(\hat{c}_{i}=\hat{c}_{J}|\cdot)\} + 1 - Pr(\hat{c}_{J}=\hat{c}_{J}|\cdot)=\Delta p_d
    \end{equation}
\end{proof}

\begin{thm}[Nearly latent indistinguishable] 
Let two samples be \textit{nearly latent indistinguishable}, $Pr(\hat{c}_i=\hat{c}_j|d_i,d_j;f)\ge \alpha \ge 0$. Then there exists the unique $\alpha^*$ that satisfies $Pr(\hat{c}_i=\hat{c}_j|d_i,d_j;f)\ge \alpha \leftrightarrow f(x_{i})^Tf(x_{j})\ge \alpha^*$.
\vspace{-0.5em}
\end{thm}

\begin{proof}
    By the definition of $f$, there exists $\theta \in \left[0,\pi\right]$ that satisfies $\cos\theta=f(x_{i})^Tf(x_j)$.
    \begin{equation}
    h(f(x_{i})^Tf(x_{j}))=h(\cos\theta)
    \end{equation}
    Note that cosine function is continuous and monotonic decreasing in the inverval $\left[0,\pi\right]$. Then,
    \begin{equation}
    h(-1)=h(\cos(\pi))=0,h(1)=h(\cos(0))=1.
    \end{equation}
    By the property of the continuous function, $(h\circ \cos)$ is continuous and monotonic decreasing function. Let $Pr(\hat{c}_i=\hat{c}_j|\cdot)=\alpha$ and $0\le \alpha \le 1$. Then we can get a unique $\theta^*$ that satisfies $h(\cos(\theta^*))=\alpha$ by using intermediate value theorem. If we set $\alpha^*=\cos(\theta^*)$, the proposition is proved.
\end{proof}

\subsection{Details of Experiments}
\label{apdx-detail-5}

\textbf{Training setting.} To validate that the proposed frameworks are robust with biased dataset, we design several experiments. We use only the CNN layers of the ResNet-50 \citep{he2016deep} as the backbone. We modify the first CNN kernel like \citet{pmlr-v119-chen20j}. We add just one linear layer for the projection layer. The dimension of the projected vectors is 256. We use Adam optimizer \citep{kingma2014adam} and Cosine scheduler \citep{chen2020improved} for decaying learning rate. Initial learning rate is set to 0.05. We trained for 40k steps with the same batch size 256 for each experiment. The temperature parameter $\tau$ is set to 0.7 for all experiments. We use fixed, weak augmentations which contain only color jittering, grayscale and horizontal filp. As images are small, we do not apply cropping or resizing to the images.

\textbf{Implementation details.} To implement our proposed models, we need to connect the memory structure to previous frameworks while not harming their properties. Designing D-MoCo is relatively easier than D-SimCLR because MoCoV2 already has the memory in it. We simply add the policy to the model directly. We use memory for D-MoCo which can hold 2048 representations of images. However, in D-SimCLR, the memory needs to hold raw images instead of their representations because the model does not use the momentum encoder. Also, for the memory efficiency, we use the stop gradient to the representations from the memory. D-SimCLR contains memory with 512 images. Figure \ref{figfunction} shows the shape of score functions we use for the experiments. We use a gaussian kernel $h(x)=\exp(-(x-1)^2/\tau)$ with temperature parameter $\tau=1$ and normalize it to fit $h(x)\in[0,1]$.

\textbf{Evaluation metrics.} We evaluate trained models with several metrics. Firstly, we use linear probing as the downstream task. We train the linear layer with SGD optimizer for 100 epochs each. The momentum for SGD is 0.9 and weight decay parameter is $10^{-6}$. Initial learning rate for this task is 0.01. We also measure the Top-$k$ accuracies with different $k$-s (1,10).

\textbf{Resource usage.} Table \ref{table3} shows the usage of time and VRAM for both the model and the memory $\mathcal{M}$. Note that D-MoCo only needs a small amount of memory and some affordable time, contrary to its robustness against biases. D-SimCLR needs additional time because it stores images instead of their representations and it should extract representations again with newly updated encoder.

\begin{table}[h]
    \renewcommand*{\arraystretch}{1.1}
    \centering
    \begin{minipage}{0.35\textwidth}
        \vspace{-6em}
        \caption{Resource usage.}
        \label{table3}
        \begin{adjustbox}{width=\textwidth}
            \begin{tabular}{c|cc}
            \Xhline{3\arrayrulewidth}
            \multirow{2}{*}{ Methods} & Time & Memory\\
             & ($h$) & (MB)\\
            \hline
            MoCoV2 & 2.7 & 10,415 \\
            SimCLR & 4.1 & 15,573 \\
            SimSiam & 4.0 & 15,541 \\
            \hline
            D-MoCo & 3.4 & 10,415 \\
            D-SimCLR & 7.7 & 19,065 \\
            \Xhline{3\arrayrulewidth}
            \end{tabular}
        \end{adjustbox}
    \end{minipage}%
    \hspace{2em}
    \begin{minipage}{0.5\textwidth}
        \begin{adjustbox}{width=\textwidth}
            \includegraphics{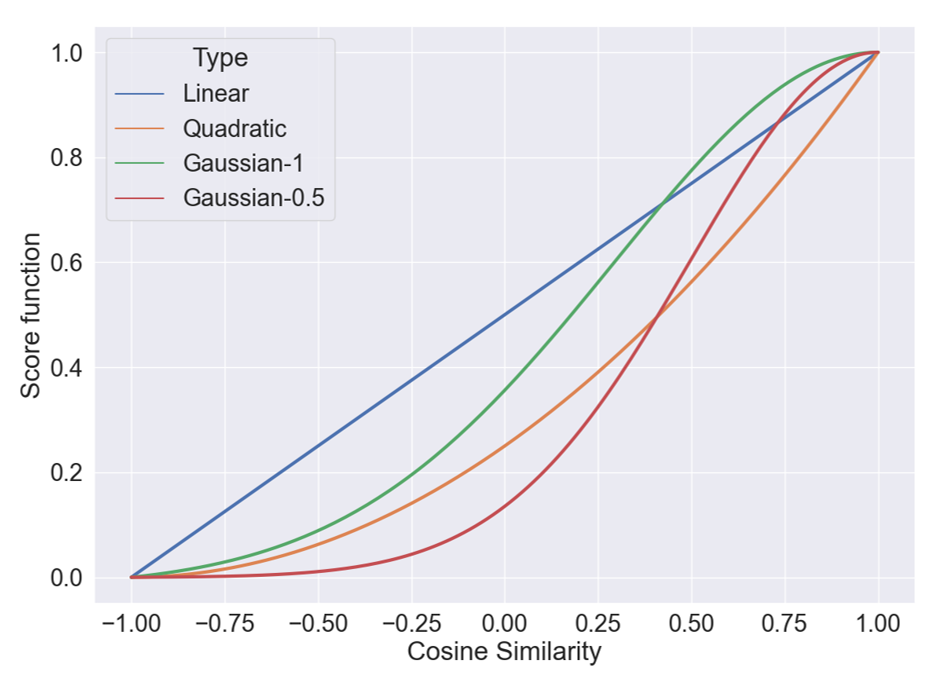}
        \end{adjustbox}
        \captionof{figure}{Visualization of score functions.}
        \label{figfunction}
    \end{minipage}
    \vspace{-0.5em}
\end{table}

\end{document}